%% file: main.tex
\documentclass[letterpaper, 10 pt, conference]{ieeeconf}  

\IEEEoverridecommandlockouts                          
\usepackage{cite}
\usepackage{amsthm}
\usepackage{amsmath,amssymb,amsfonts,bm}
\usepackage{algorithm}
\usepackage{algpseudocode}

\usepackage{graphicx}
\usepackage[export]{adjustbox}
\usepackage{textcomp}
\usepackage{xcolor}

\usepackage{MnSymbol} 
\usepackage{latexsym,color,minipage-marginpar,caption,multirow,verbatim}
\usepackage{caption}
\usepackage{subcaption}
\usepackage{enumerate}
\usepackage{hyperref}
\usepackage{tkz-tab}
\usepackage{float}
\usepackage{enumitem}
\usepackage{adjustbox}
\usepackage{pifont}
\usepackage{makecell}
\usepackage{stfloats}
\usepackage{relsize}

\usepackage{color, colortbl}
\usepackage{bbm}
\definecolor{Gray}{gray}{0.8}
\definecolor{LightGray}{gray}{0.95}

\input{macros}

\begin{document}

\title{\huge Conformal Decision Theory: \\ Safe Autonomous Decisions from Imperfect Predictions
}

\author{Jordan Lekeufack$^{1,*}$ 
Anastasios N. Angelopoulos$^{2,*}$  
Andrea Bajcsy$^{3,*}$\thanks{$^1$Department of Statistics, UC Berkeley.
$^2$Department of Electrical Engineering and Computer Science, UC Berkeley. Emails: \texttt{\{jordan.lekeufack, angelopoulos, michael\_jordan, malik\}@berkeley.edu} $^3$Robotics Institute, Carnegie Mellon University. Email: \texttt{\{abajcsy\}@cmu.edu} $^*$ Equal contribution. $^{**}$ Equal senior authorship. }  %
Michael I. Jordan$^{1,2,**}$
Jitendra Malik$^{2,**}$
}

\makeatletter
\let\@oldmaketitle\@maketitle
\renewcommand{\@maketitle}{\@oldmaketitle
\centering
\includegraphics[width=\textwidth]{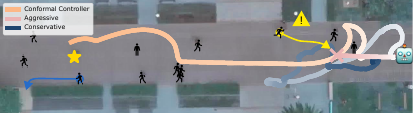}
\captionof{figure}{\fontsize{9}{10.5}\selectfont Robot planner using a conformal controller on the Stanford Drone Dataset \cite{robicquet2016learning}. The future trajectories of humans are predicted online by a machine learning algorithm (not visualized). The robot planner finds an optimal spline through the scene and is penalized for being close to humans. 
This penalty is proportional to a conformal control variable, $\param_t$, which is adjusted online by the conformal controller so the average distance from a human is no less than two meters. 
The orange, red, and blue curves are the robot trajectory with different planners: the conformal controller, an aggressive planner with $\param=0$ (i.e., no reward for avoiding humans), and a conservative planner with a large negative value of $\param$ (i.e., a large reward for avoiding humans). 
The darkness of the lines indicates the passage of time. 
Illustrative pedestrian trajectories are plotted as arrows; only the yellow pedestrians affect the spline planner. Details in Section~\ref{subsec:sdd} and videos on \href{https://conformal-decision.github.io}{project website}$^\dag$ .}
\label{fig:front_fig}
\vspace{-0.2in}
\bigskip}
\makeatother
\maketitle

\thispagestyle{plain}
\pagestyle{plain}

\blfootnote{$^\dag$ https://conformal-decision.github.io}

\begin{abstract}
\input{abstract}

\end{abstract}

\input{introduction}

\input{related_work}
\input{conformal_decision_theory}

\input{theory}
\input{experiments}

\input{discussion}

\bibliographystyle{IEEEtran}
\bibliography{biblio}
\end{document}

%% file: macros.tex
\theoremstyle{plain}
\newtheorem{theorem}{Theorem}
\newtheorem{lemma}{Lemma}[theorem]

\newtheorem{corollary}[theorem]{Corollary}

\newtheorem{definition}{Definition}

\newtheorem{remark}{Remark}

\newcommand{\safe}{{\rm safe}}

\newcommand{\E}{\mathbb{E}}

\newcommand{\R}{\mathbb{R}}

\newcommand{\ep}{\varepsilon}
\renewcommand{\ast}{*}

\newcommand{\statespace}{\mathcal{X}}
\newcommand{\actionspace}{\mathcal{U}}
\newcommand{\targetspace}{\mathcal{Y}}
\newcommand{\state}{x}
\newcommand{\action}{u}
\newcommand{\target}{y}
\newcommand{\param}{\lambda}
\newcommand{\thresh}{\ep}
\newcommand{\dfam}{\mathcal{D}}
\newcommand{\dparam}{D^{\param}}

\newcommand{\loss}{\mathcal{L}}
\newcommand{\setT}{[T]}
\newcommand{\paramlow}{\param^{\rm safe}}
\newcommand{\cmark}{{\color{green}\ding{51}}}
\newcommand{\xmark}{{\color{red}\ding{55}}}

\newcommand{\reward}{J}

\def\BibTeX{{\rm B\kern-.05em{\sc i\kern-.025em b}\kern-.08em
T\kern-.1667em\lower.7ex\hbox{E}\kern-.125emX}}

\newcommand{\para}[1]{\smallskip \noindent \textbf{{#1}.}}

\newcommand\blfootnote[1]{%
  \begingroup
  \renewcommand\thefootnote{}\footnote{#1}%
  \addtocounter{footnote}{-1}%
  \endgroup
}

\hypersetup{
    colorlinks=true,
    linkcolor=black,
    filecolor=magenta,      
    urlcolor=cyan,
    citecolor=black,
    pdftitle={Conformal Decision Theory},
    pdfpagemode=FullScreen,
    }

%% file: abstract.tex
We introduce \emph{Conformal Decision Theory}, a framework for producing safe autonomous decisions despite imperfect machine learning predictions.
Examples of such decisions are ubiquitous, from robot planning algorithms that rely on pedestrian predictions, to calibrating autonomous manufacturing to exhibit high throughput and low error, to the choice of trusting a nominal policy versus switching to a safe backup policy at run-time. 
The decisions produced by our algorithms are safe in the sense that they come with provable statistical guarantees of having low risk without any assumptions on the world model whatsoever; the observations need not be I.I.D. and can even be adversarial.
The theory extends results from conformal prediction to calibrate decisions directly, without requiring the construction of prediction sets.
Experiments demonstrate the utility of our approach in robot motion planning around humans, automated stock trading, and robot manufacturing.

%% file: introduction.tex
\section{Introduction}
\label{sec:intro}

Autonomous systems increasingly rely on complex learned models to supply predictions that are the basis for decision-making. Self-driving cars rely on deep neural networks \cite{alahi2016social, jain2016structural, vemula2018social, salzmann2020trajectron++} to plan paths around nearby pedestrians, robotic manipulators leverage learned grasp models \cite{mahler2017dex} to plan high-throughput pick-and-place maneuvers in factories, and AI-enabled trading agents optimize the financial future of investors \cite{yin2022graph}. There is a conceptual gap between prediction and decision-making, and it remains a challenge to ensure that
systems make \emph{good decisions} despite \emph{imperfect predictions}. 

One increasingly popular strategy is to quantify the uncertainty in the predictions independently of their downstream effect on the decision via conformal prediction (CP)~\cite{vovk2005algorithmic, vovk2018conformal, gibbs_adaptive_2021, zaffran2022adaptive, angelopoulos2021gentle, feldman2023achieving}.
This approach has become popular because, when used to provide simultaneous prediction sets on all outcomes, conformal prediction provides statistical guarantees of safe autonomous behavior without any assumption on the underlying distribution or model.
This application of CP has shown impact in robot navigation \cite{chen2021reactive, lindemann2023safe, dixit2023adaptive, muthali2023multi}, early warning systems \cite{luo2022sample}, out-of-distribution detection \cite{cai2020real, sinha2023closing}, probabilistic pose estimation \cite{yang2023object}, and for large language models  \cite{ren2023robots}.
However, the requirement of simultaneous coverage is challenging to satisfy and for many decision systems is excessive.
What if we could provide statistical guarantees, as in CP, \textit{directly} on our decisions, bypassing the need to construct prediction sets?


This work presents Conformal Decision Theory, a theoretical and algorithmic framework that unifies predictive uncertainty and safe decision-making. Our key idea is 
\begin{quote}
     \textit{instead of calibrating  \textbf{prediction sets} for coverage, we directly calibrate \textbf{decisions} for low risk. }
\end{quote}
Our main algorithmic innovation is a class of algorithms called \emph{conformal controllers}.
A conformal controller starts with a conformal control variable, $\lambda_t$, which determines the decision-maker's conservatism or aggressiveness.
Then, it dynamically adjusts $\lambda_t$ to balance risk and performance in such a way that guarantees a low risk.
The main practical benefit of this approach is its \emph{emergent ability to ignore irrelevant uncertainty}, only accounting for that which \emph{affects decisions}. 
This can be much less conservative than the prediction-set strategy.
For example, in Figure~\ref{fig:front_fig}, the planner only considers the humans that pose a collision risk. 

The contributions of this paper are threefold:
\begin{itemize}
    \item We introduce Conformal Decision Theory, the idea of directly calibrating decisions with conformal controllers. This extends the line of work in online adversarial conformal prediction~\cite{gibbs_adaptive_2021, bastani2022practical, feldman2023achieving, angelopoulos2023conformal} to the decision-making setting. 
    \item We prove finite-time risk bounds for conformal controllers. Even when applied to prediction sets, these results are stronger than any previously known results for online adversarial conformal prediction. 
    \item We show the utility of the framework in three simulations where Conformal Decision Theory is applied to robot navigation: the Stanford Drone Dataset \cite{robicquet2016learning}, a stock trading simulation, and a robot manufacturing example. 
\end{itemize}
The main potential impact of this work is to broaden the scope of conformal prediction.
Our methods are more appropriate for disciplines that focus on decision-making, such as control theory, reinforcement learning, and logistics.
In these disciplines, algorithms are ultimately evaluated by the decisions, not the predictions, that they make.
Furthermore, there are many settings where it does not make sense to construct prediction sets, and our technique can provide a  distribution-free outlook for such problems (see, e.g., Section~\ref{subsec:factory}).

%% file: related_work.tex
\section{Related Work}
\label{sec:related_work}

\para{Decision-Making Under Predictive Uncertainty} 
Within the machine learning and statistics community, uncertainty quantification of prediction models has been studied widely, from conformal prediction to Bayesian neural networks to ensembles \cite{gal2016dropout,lampinen2001bayesian, goan2020bayesian, lakshminarayanan2017simple, koenker1978regression}. 
Instead of focusing on prediction calibration alone, the controls and optimization community have coupled prediction uncertainty with safe (i.e., risk bounded) decision-making via
chance-constrained optimal control \cite{du2011robot, bujarbaruah2020adaptive} and scenario optimization \cite{campi2009scenario, de2023scenario}. The former typically constructs prediction sets that are used as constraints while the latter safeguards against samples drawn from the prediction model. 
Instead of directly calibrating the output of upstream prediction modules or solving decision-making problems under probabilistic constraints, this work presents a theoretical and algorithmic approach to tuning the robot's decision risk directly as a function of historical decision-making performance. 

\para{Online Learning \& Nonstochastic Control}
The method herein is reminiscent of online learning, and specifically the online gradient descent (OGD) update of~\cite{zinkevich2003online}.
The connection is most apparent when examining the forthcoming Equation~\eqref{eq:CC} with $\ell_t = \mathbf{1} \{y_t \in \mathcal{C}(x_t)\}$; 
this recovers the ACI algorithm of~\cite{gibbs_adaptive_2021}, which is OGD on the quantile loss~\cite{koenker1978regression}.
However, the update in~\eqref{eq:CC} is substantially more general because it incorporates arbitrary decision rules, and reframing it as OGD on an analytic loss function is generally impossible.
Furthermore, the guarantees in~\cite{gibbs_adaptive_2021} are not to our knowledge recoverable by existing regret analyses from online convex optimization and nonstochastic control, e.g., ~\cite{hazan2016introduction,bubeck2011introduction, hazan2022introduction}.
However, the guarantees do share a retrospective flavor, in that, like regret analyses, they provide guarantees on average over the observed history.

%% file: conformal_decision_theory.tex
\section{Conformal Decision Theory}
\label{sec:theory}

Conformal Decision Theory (CDT) is an approach for calibrating an agent’s decisions to achieve statistical guarantees for the realized average loss of those decisions.
Consider a decision-making agent whose input space is $\statespace$ and action space is $\actionspace$. 
In our running example of robot navigation, $\state_t \in \statespace$ captures the current state of the robot, the current scene information (e.g., environment geometry), and the agent information (e.g., pedestrian predictions) while $\action_t \in \actionspace$ is the action that the ego vehicle plans at the current time $t$. 
At time $t$, the agent has access to a family of \emph{decision functions}
\begin{equation}
\label{eq:decision_fam}
\dfam_t := \left\{\dparam_t : \statespace \rightarrow \actionspace, \param \in \mathbb{R}\right\},
\end{equation}
parameterized by $\param$, which we call a \emph{conformal control variable}.
One should think of $\param$ as indexing the decisions from least to most conservative.
In Figure~\ref{fig:front_fig}, $\dfam_t$ is the set of dynamically feasible splines at time $t$, $\param$ is the coefficient of the reward term for avoiding humans, and $\dparam_t$ is the spline maximizing the total reward given $\param$.

Assessing the quality of an agent's decision depends on a space of \textit{targets} $\targetspace$. Importantly, the realizations of these targets are \emph{unknown} at the time of the decision; the agent only observes them at deployment time, after decisions are made, and in an online fashion. For example, the robot in Figure~\ref{fig:front_fig} does not know the true future state of nearby pedestrians; at any current time $t$, it only knows the (potentially erroneous) pedestrian predictions. In this example, $\targetspace$ is the space of pedestrian states (e.g., 2D positions) and $\target_t \in \targetspace$ is the \textit{true} state that the pedestrian moves to at time $t$.

Mathematically, the quality of the decision-making is quantified by a \emph{loss function} $\loss: \actionspace \times \targetspace \rightarrow [0, 1]$.\footnote{The framework works for any bounded loss, but we assume the loss to be in $[0,1]$ for simplicity.}
Often, the loss is more likely to be large when aggressive decisions are taken---i.e., when $\param$ is large.
Aggressive decisions can be unsafe, but taking $\param$ too small yields conservative and under-performing decisions. 

We seek an algorithm for adapting $\param_t$ (and thus the corresponding decision $\dparam_t$) at each time step such that the average loss is controlled in hindsight for \textit{any} realization of an input-target sequence $\{(x_t, y_t)\}^T_{t=1}$.
This is commonly known as the \emph{adversarial sequence model}~\cite{dawid1982well, gibbs_adaptive_2021}.
Here, our goal is to set $\param_{1:T}$ to achieve a \emph{long-term risk bound}:
\begin{equation}
    \text{find } \param_{1:T} \text{ s.t. } \hat R_T(\dfam_{1:T}, \param_{1:T}) \leq \thresh + \frac{C \cdot h(T)}{T},
    \label{eq:risk_bound_cdt}
\end{equation}
where $\thresh$ is a pre-defined risk level in $[0, 1]$, $C$ is a (small) constant, $h(T)$ is any sublinear function; i.e., one where $h(T)/T \to 0$ as $T\to \infty$, and 
\begin{equation}
    \hat R_T(\dfam_{1:T}, \param_{1:T}) := \frac{1}{T} \sum_{t=1}^T \loss(D^{\param_t}_t(\state_t), \target_t) \quad {\rm and} \quad \hat R_0 = 0.
\end{equation}

We will omit $\dfam_{1:T}$ in the notation of risk when the sequence of family of decision functions is clear from the context


%% file: theory.tex
\section{Theory \& Conformal Controller Algorithm}



In this section, we prove the core theoretical results behind Conformal Decision Theory. Specifically, we show that any sequence of families of decision functions $\dfam_{1:T}$ that are \textit{eventually safe} can be calibrated online to achieve bounded long-term risk. 
We then introduce an example of conformal controller which solves Equation~\eqref{eq:risk_bound_cdt} under the assumption of eventual safety. 

\begin{definition}[Eventually Safe] \label{def:eventually-safe}
    In the setting above, we say that $\dfam_{1:T}$ is \textit{eventually safe} if $\exists \; \thresh^\safe \in [0, 1], \paramlow \in \R$ and a time horizon $K>0$ such that uniformly over all sequences $\lambda_{1:K}$ and $\{(\state_1,\target_1), \dots , (\state_k,\target_k)\} \in \statespace \times \targetspace$,
    \begin{equation}
        \begin{aligned}
        \big\{\forall k\in [K], \param_k & \leq \paramlow\big\} \\
        & \Longrightarrow \frac{1}{K} \sum_{k=1}^K\loss\left(D^{\param_k}_k(\state_k), \target_k\right) \leq \thresh^\safe.
        \end{aligned}
    \end{equation}
\end{definition}
Intuitively, this condition says that there exists a safe value $\paramlow$ such that if the conformal control variable lands below that value, it will incur a low risk $\thresh^\safe$ after no more than $K$ time steps.
For example, even the most conservative robot planner may not be able to change its trajectory fast enough \emph{in a single timestep}, but it could possibly do so in $K$ time steps. 
For general decision-making, the existence of a safe decision function is not guaranteed, and requires domain-specific knowledge (e.g., when the loss function captures the distance between agents \cite{bansal2017hamiltonjacobi, kim2023datadriven, hsu2023safety}). 
But when the decision is a \textit{prediction set}, a safe decision function is trivial because you can always output the entire space.
Note that the \textit{eventually safe} is a strictly weaker assumption than that used for the proofs in other works, such as~\cite{gibbs_adaptive_2021,bhatnagar2023improved, angelopoulos2023conformal}, which require $K=1$. Moreover,
conformal controllers are simple yet efficient algorithms that solve the Conformal Decision Theory problem stated in Equation~\eqref{eq:risk_bound_cdt}.
An example is below.

\begin{theorem}[Conformal Controller]
\label{prop:CC}
Consider the following update rule for $\lambda_{1:T}$:
\begin{equation}
    \param_{t+1} = \param_{t} + \eta \left( \thresh - \ell_t \right), \forall t \in [T]
    \label{eq:CC}
\end{equation}
where $\eta > 0$ and $\ell_t := \loss(D_t^{\param_t}(x_t), y_t)$. 

If  $\param_1 \geq \paramlow - \eta$ and $\dfam_{1:T}$ satisfies Definition~\ref{def:eventually-safe} for a given $K\geq1$ and $\thresh^\safe \leq \thresh$, then for any realization of the data, the empirical risk is bounded:
\begin{equation}
     \hat R_t(\param_{1:t}) \leq  \thresh +  \frac{(\param_1 - \paramlow)/\eta + K}{t},
\end{equation}
for all $t \in [K,..., T]$.
\end{theorem}
The update in~\eqref{eq:CC} resembles ACI~\cite{gibbs_adaptive_2021} and is a hybrid between the RollingRC update~\cite{feldman2023achieving}, and the P-controller update~\cite{angelopoulos2023conformal}.
The difference is that the update is applied to $\param$ and not the conformal quantile or quantile level.

\begin{proof}[Proof of \autoref{prop:CC}]
    By the definition of the update rule,
     \begin{equation}
        \param_{t+1} = \param_1 + \eta \sum\limits_{s=1}^{t} (\thresh - \ell_s).
    \end{equation}
    By isolating $\sum_{s=1}^t \ell_s$ on one side and moving all other terms to the right-hand side, we obtain:
    \begin{align*}
        \hat R_t(\param_{1:t}) = \frac{1}{t} \sum\limits_{s=1}^{t} \ell_s &= \thresh +  \frac{\param_1 - \param_{t+1} }{\eta t}. 
    \end{align*}
    To conclude, we just need to lower bound $\lambda_t$ by a constant w.r.t $t$ which is done in the following Lemma \ref{lemma:param_bounds}.
\end{proof}

\begin{lemma}
    \label{lemma:param_bounds}
    For the sequence in Equation ~\ref{eq:CC}, with $\param_1 \geq \paramlow - \eta$ we have that the parameter $\param_t$ is bounded below by $\param_t \geq \paramlow - K\eta$, for all $t \in [T+1]$. 
\end{lemma}

\begin{proof}
    First note that the maximal change in the parameter is $\sup_{s \in\setT} |\param_{s+1} - \param_s| < \eta$, because $\ell_s \in [0,1]$ and $\thresh \in [0,1]$. We will then proceed by contradiction: Assume that $\inf_{s\in [T+1]} \param_s < \paramlow - K\eta$. Denote $t = \arg\min_{s\in[T+1]} \{s : \param_s < \paramlow - K\eta\}$.  That is, $t$ is the first instant when the parameter goes below that lower bound. Then, by definition of $t$, $\forall s<t, \param_t <\paramlow - K \eta \leq \lambda_s$.
    
    Because the max difference between successive steps is $\eta$, we can prove recursively that $\forall k \in \{0,\dots, K\}, \param_{t-k} < \paramlow - (K-k)\eta$. Note that, from those inequalities, we deduce that $t>K$ since $\param_1 \geq \paramlow - \eta$. By recursively applying the update rule $\param_t = \param_{t-K} + K\eta (\ep - \frac{1}{K}\sum_{k=1}^K \ell_{t-k})$, we have:
        \begin{align*}
        &\left(\forall k \in \{0,\dots, K-1\}, \; \param_{t-k} < \paramlow  \right) \\
        &\Longrightarrow \frac{1}{K} \sum_{k=1}^K \ell_{t-k} \leq \thresh^\safe \kern 6pc \text{(Definition \ref{def:eventually-safe})}\\
        &\Longrightarrow \param_t = \param_{t-K} + K\eta \left(\ep - \frac{1}{K}\sum_{k=1}^K \ell_{t-k}\right)  \\
        &\Longrightarrow \param_t \geq \param_{t-K} + K\eta(\thresh-\thresh^\safe) \\
        &\Longrightarrow \param_t \geq \param_{t-K}.
    \end{align*}
    
    Since $t$ is the first ever timestep to go below $\paramlow - K\eta$, this is a contradiction.
\end{proof}

\begin{remark}
    The assumption $\param_1 \geq \paramlow - \eta$ is not necessary to prove that $\hat R_t(\lambda_{1:t}) \leq \ep + O(1/t)$. Intuitively, two scenarios can occur:
\begin{enumerate}
    \item If $\forall k \in [K], \param_k \leq \paramlow$, then the empirical risk over the first $K$ steps will be upper bounded by $\thresh^\safe$. In this case, we need only to upper bound the risk between $K+1$ and $T$, which can be achieved using the previous theorem or this remark.
    \item If there exists a $k \in [K]$ such that $\param_k > \paramlow$, we can apply the previous theorem to upper bound the risk between $k$ and $T+1$. The cumulative loss between $1$ and $k$ is upper bounded by $k$, which is $o(1)$.
\end{enumerate}

\end{remark}

\subsection*{Conformal Decision Theory in Batch}
Conformal decision theory can also be applied in the so-called batch setting, wherein a separate calibration dataset is available for learning a safe decision.
Here, a dataset or simulator allows for offline experimentation to quantify the risk of different decisions, e.g., offline RL. This requires a different statistical setup. 
Consider the case of $n+1$ exchangeable decision functions $D_1(\param), \ldots, D_{n+1}(\param)$ and an associated loss function $\loss$ taking a decision and returning a value in $[0,1]$.
The first $n$ decision functions will be used for calibration of a parameter $\hat{\lambda}$ that will be used in the final decision.
These exchangeable decision functions may be produced, for example, by applying a single decision function to a sequence of exchangeable data points.
For the sake of simplicity, we assume that the decisions have monotone loss, i.e., that for all $i$, 
\begin{equation}
    \lambda_1 \leq \lambda_2 \implies \loss\left(D_i(\lambda_1)\right) \leq \loss\left(D_i(\lambda_2)\right).
\end{equation}
Following~\cite{angelopoulos2022conformal}, the conformal control variable can be chosen as
\begin{equation}
    \hat{\lambda} = \sup\left\{ \lambda : \frac{1}{n}\sum\limits_{i=1}^n \loss\left(D_i(\lambda)\right) \leq \epsilon - \frac{1-\epsilon}{n} \right\}.
\end{equation}
This will give a risk guarantee as a corollary of Theorem 1 of~\cite{angelopoulos2022conformal}.
\begin{corollary}
With the choice of $\hat{\lambda}$ above,
\begin{equation}
    \E\left[ \loss\left(D_{n+1}(\hat{\lambda})\right) \right] \leq \epsilon.
\end{equation}
\end{corollary}

Though the validity of the algorithm follows from the theory of conformal risk control, it is substantially different in practice and deserves further study.
Specifically, unlike the previous methods, in order to calculate $\hat{\lambda}$, one must iterate through a sequence of counterfactual decisions (possible values of $\lambda$) and evaluate what the loss would have been.
This restricts the applications of the batch algorithm and also presents an opportunity for future work to make it more efficient and expand its scope.

%% file: experiments.tex
\section{Experiments}
\label{sec:experiments}

We demonstrate Conformal Decision Theory in three autonomous decision-making domains, which exhibit three different ways in which a conformal controller can be instantiated.  
First, we consider a robot-navigation-around- humans example in the Stanford Drone Dataset \cite{robicquet2016learning}, where CDT tunes the robot's reward function in an online manner to be safe but efficient. 
Next, we model a manufacturing setting where CDT directly calibrates the speed of the conveyor belt under a robot to achieve high-throughput and successful robot grasps. 
Finally, we study an automated high-frequency trading example where CDT must optimize the buying and selling of stocks.

\input{exp_sdd}

\input{exp_factory}

\input{exp_trading}

%% file: exp_sdd.tex
\subsection{Robot Navigation in Stanford Drone Dataset}
\label{subsec:sdd}

Robot navigation around people must balance safety (i.e., not colliding with humans) and efficiency (i.e., the robot makes progress towards a goal). To ensure that the risk of collision is low while still making progress to the goal, the robot will calibrate its cost function at run-time using a conformal controller (CC).
%

\input{nexus_4_darts}

\begin{figure*}[t!]
    \centering
    \includegraphics[width=\textwidth]{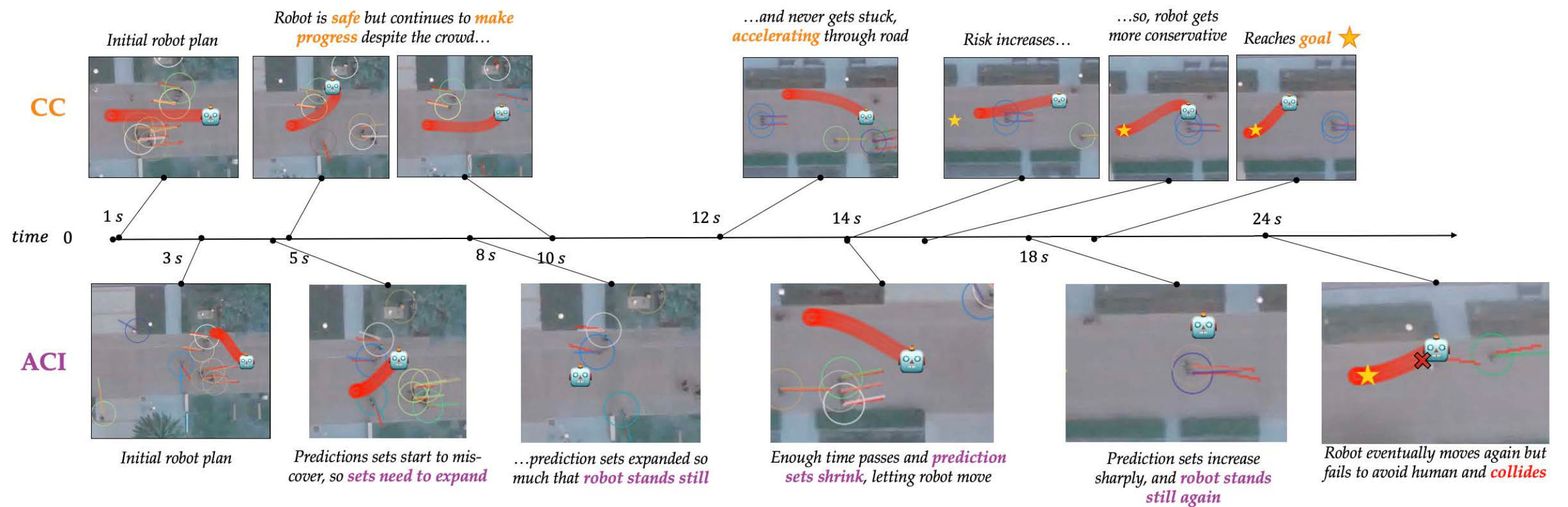}
    \caption{\textbf{Stanford Drone Dataset: Qualitative Results.} Visualization of interaction over time (left to right). (Top) With our conformal controller (CC), the robot always makes progress towards its goal while remaining safe, even when blocked by crowds of people. (Bottom) The ACI baseline calibrates the prediction sets. As soon as a mis-prediction happens, ACI expands the prediction sets to obtain coverage, but this frequently blocks the robot from moving anywhere (see $t=10 s$), even though the mis-predictions occurred for a pedestrian who was far away and not interfering with the robot's plan.}
    \label{fig:sdd_qualitative}
    \vspace{-1em}
\end{figure*}

\para{Decision Function \& Parameterization} 
The robot plans via model predictive control, where at each timestep it fits a minimum-cost spline subject to its dynamic constraints, which are modeled as a nonlinear Dubins car \cite{walambe2016optimal}. Let $g := [g_x, g_y] \in \mathbb{R}^2$ be the robot's goal location. 
Let $t$ be the current time, $H < T$ be the planning horizon, and $\action_{t:t+H} \in \mathbb{R}^{H \times 3}$ be a spline consisting of the robot's planar position and orientation.
The robot also gets as input the current set of short-horizon predictions of each human's state, $\state_{t:t+H} \in \mathcal{P}_t$, generated by an autoregressive predictive model \cite{hyndman2018forecasting}. Note that this set $\mathcal{P}_t$ can include predictions for \textit{multiple} humans in the scene (as shown in Figure~\ref{fig:front_fig}). 
The robot's planning objective is 
\begin{equation}
    \reward(\action_{t:t+H}; \mathcal{P}_t, \param) := \mathlarger{\sum}\limits_{\tau=t}^{t+H} \underbrace{\vphantom{\inf_{\substack{x_\tau \in \mathcal{P}_t }}}\| \action_\tau^{\mathrm{pos}} - g \|}_{\text{Goal distance}} + \param \cdot \underbrace{(-\inf_{\substack{x_\tau \in \mathcal{P}_t }} \| \action_\tau^{\mathrm{pos}} - \state_\tau \|)}_{\text{Human avoidance}},
\end{equation}
where the notation $\action_\tau^{\mathrm{pos}} \in \mathbb{R}^2$ indicates the $xy$-positional entries of the robot's state at time $\tau$. Note that the conformal control variable $\param$ scales the cost of staying far away from predicted human states: if $\param = 0$ the robot only cares about reaching the goal; if $\param > 0$ then the robot is increasingly penalized for intersecting with predicted human trajectories.
The decision function outputs the minimum-cost trajectory for the robot
\begin{equation}
    \dparam_t := \arg \min_{\action_{t:t+H} \in \actionspace} \reward(\action_{t:t+H}; \mathcal{P}_t, \param),
\end{equation}
where $\actionspace$ is the set of feasible splines (ones that are dynamically feasible for the robot and also do not intersect with environment obstacles). At the next timestep, the robot re-predicts the human trajectory (i.e., generates $\mathcal{P}_{t+1}$) and re-plans the decision $\dparam_{t+1}$.

\para{Loss Function} 
Let $\targetspace \subset \mathbb{R}^2$ and the targets $\target_t^1,\dots, \target_t^M \in \targetspace$ be the actual $xy$ positions of each of the $M$ humans that the robot observes at time $t$. The loss function is defined as the negative distance to the nearest human,
\begin{equation}
    \loss := -\inf_{i \in [M]} \| \target^i_t - \action^{\mathrm{pos}}_t \|_2,
\end{equation}
where $\action^{\mathrm{pos}}_t$ is the robot's current position. To make this value bounded, we clip the loss to the size of the video. Note that because  we use a negative loss, we also changed $\param$ so that the larger $\param$, the more conservative the decision.

\para{Metrics} 
We measure a boolean \textit{safe} variable indicating if the robot did not ever collide with a human. 
We also measure a boolean \textit{success} variable if the robot reached the goal location by the end of the interaction episode (i.e., length of video in the dataset).
We also measure the time to reach the goal location and the minimum, mean, and $\{5\%, 10\%, 25\%, 50\%\}$ quantiles of the distance to the nearest human.

\para{Experimental Setup} All methods are evaluated on interactions from the \texttt{nexus\_4} video in the Stanford Drone Dataset (SDD) \cite{robicquet2016learning}. The risk threshold is $\thresh=2 m$ (i.e., radius around human). The robot always starts from the same initial condition and moves to the same goal. This scenario has a high density of pedestrians, making the risk-performance tradeoff for the robot nontrivial. Our approach (\textbf{CC}) adapts the reward weight $\param_t$ on the human collision cost based on \autoref{eq:CC} so that the decision risk is calibrated. Our baseline robot planners: \textbf{conservative} which always uses the safe decision function $D_t^{\param = 1}$, \textbf{aggressive} which uses $D_t^{\param = 0}$, and \textbf{ACI} \cite{dixit2023adaptive} which first uses adaptive conformal prediction to calibrate prediction sets and then plans to avoid these sets.

\begin{figure}[t!]
    \centering
    \includegraphics[width=0.5\textwidth]{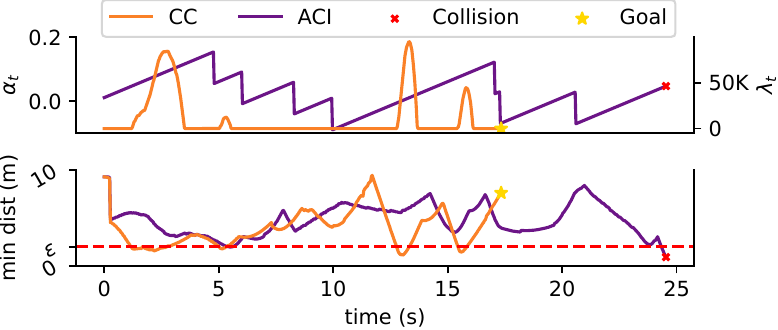}
    \caption{\textbf{Stanford Drone Dataset.} (Top) Trajectories of $\param_t$ (calibrated by \textbf{CC}) and $\alpha_t$ (from \textbf{ACI} to calibrate sets). When $\alpha_t \leq 0$, ACI returns infinite set and the robot stops. (Bottom) Distance to the nearest human over time. $\param_t$ is large when the robot is close to human, while $\alpha_t$ is unrelated. The $\param_t$ trajectory is shorter because it reaches the goal faster.}
    \label{fig:sdd_lineplots}
    \vspace{-1.5em}
\end{figure}

\para{Results} 
Quantitative results shown in Table~\ref{tab:sdd} and qualitative results in \autoref{fig:front_fig}.
Because the \textbf{conformal controller} calibrates the robot's decisions directly, it is substantially ($\sim 29 \%$) faster at reaching the goal than the \textbf{ACI} algorithm (see visualization over time in \autoref{fig:sdd_qualitative}). 
While the \textbf{aggressive} baseline reaches the goal fastest, it consistently violates the safety threshold. On the other hand, the \textbf{conservative} baseline never completes the task, getting stuck far away from the crowds of pedestrians. 
The \textbf{conformal controller} ensures safety so long as the learning rate is fast enough for the robot planner to quickly adapt to changes in nearby human behavior (see \autoref{fig:sdd_lineplots}). Note that \textbf{ACI} can result in collisions for two reasons: 1) the prediction sets do not adapt fast enough for the spline planner to react and swerve out of the way of the pedestrian, 2) if the prediction sets become so large that there is no feasible spline and the robot must stand in
place, the pedestrians sometimes run into the robot.
This issue was independently observed in~\cite{dixit2023adaptive}.

%% file: nexus_4_darts.tex
{\renewcommand{\arraystretch}{1.25}
\begin{table*}[t!]
    \centering
    \caption{\textbf{Stanford Drone Dataset: Quantitative Results.}  Results on the \texttt{nexus\_4} scenario from SDD \cite{robicquet2016learning}. The robot's goal is to cross the nexus while avoiding pedestrians. Safety was violated if the robot collided with a human. At all learning rates $\eta$, the conformal controller is more efficient at navigation than ACI in terms of time. It remains safe so long as the learning rate is set high enough so that the robot planner can quickly adapt to nearby humans; when the learning rate is set too low (near zero), proximity to humans is effectively not penalized, leading to collisions.} \label{tab:sdd}
    \begin{tabular}{|l|l|rrlllllll|}
        \rowcolor{Gray} \hline 
        \multicolumn{2}{|c|}{} & \multicolumn{9}{|c|}{\textbf{Metrics}} \\ \hline 
        \rowcolor{LightGray}
        \textbf{Method} & $\eta$ & success & time (s) & safe & min dist (m) & avg dist (m) & 5\% dist (m) & 10\% dist (m) & 25\% dist (m) & 50\% dist (m) \\ \hline 
        Aggressive & n/a & \cmark & 8.567 & \xmark & 0.1595 & 4.058 & 1.253 & 1.546 & 2.495 & 4.021 \\
        \hline
        \multirow[c]{3}{*}{\makecell{ACI \\($\alpha=0.01$)}} & 0 & \cmark & 27.17 & \xmark & 0.07612 & 5.201 & 1.842 & 2.415 & 3.9 & 5.614 \\
         & 0.01 & \cmark & 26.67 & \xmark & 0.8026 & 4.575 & 2.261 & 3.014 & 3.507 & 4.574 \\
         & 0.1 & \cmark & 24.73 & \xmark & 0.7906 & 4.771 & 2.284 & 2.825 & 3.561 & 4.78 \\
        \hline
        \multirow[c]{4}{*}{\makecell{Conformal \\Controller \\($\thresh=2$m)}} & 50 & \cmark & 20.03 & \xmark & 0.6122 & 3.299 & 0.8688 & 1.426 & 2.022 & 2.978 \\
         & 100 & \cmark & 17.4 & \cmark & 1.142 & 3.794 & 1.678 & 1.811 & 2.378 & 3.262 \\
         & 500 & \cmark & 17.33 & \cmark & 1.116 & 3.989 & 1.69 & 1.812 & 2.452 & 3.795 \\
         & 1000 & \cmark & 16.17 & \cmark & 1.265 & 3.599 & 1.698 & 1.81 & 2.282 & 3.303 \\
        \hline
        Conservative & n/a & \xmark & $\infty$ & \cmark & 2.268 & 6.291 & 3.801 & 3.982 & 4.982 & 5.993 \\
                \hline 
    \end{tabular}
\end{table*}}

%% file: exp_factory.tex
\subsection{Manufacturing Assembly Line Robot}
\label{subsec:factory}

\begin{figure*}[t!]
    \centering
    \includegraphics[width=0.85\textwidth,valign=t]{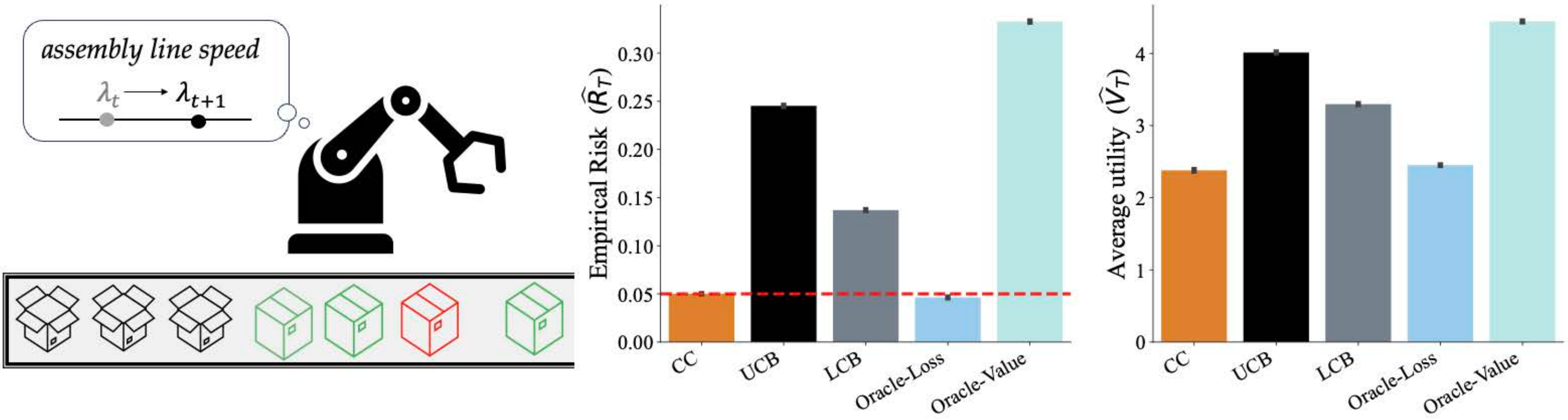}
    \caption{\textbf{Manufacturing Assembly Line Robot: Quantitative Results.} (Left) Illustrative example: Robot must adjust the speed so that it grasps the most items while minimizing grasp failure. (Right) Empirical risk ,$\hat{R}_T$, and average utility (i.e., successful grasps), $\hat{V}_T$ on 1000 runs. Our method is denoted by (CC). Dashed red line is target risk $\thresh=0.05$.}
    \label{fig:factory}
    \vspace{-1em}
\end{figure*}

Consider a factory assembly line where a robot has to grab items from a conveyor belt (left, Figure~\ref{fig:factory}). 
As the speed increases, the throughput of items increases but so does the ratio of robot grasp failures. The agent must calibrate the speed so that the ratio of failures over time stays below $\thresh$. 

\para{Decision Function \& Parameterization} 
The agent directly modifies the speed, thus the action is defined as $\action_t := \param_t$. Here we take $\param_t\in[0, 1]$.

\para{Risk Function} For a given conveyor belt speed $\param$, the robot will attempt to grab $n(\param)$ items, among which $d(\param)$ are failed grasps. The loss received by the robot will be $\loss(\param) := d(\param) / n(\param)$. 

\para{Metrics} We measure average utility (i.e., $\#$ of successful grasps), $\hat{V}_T := \frac{1}{T}\sum_{t=1}^T V(\lambda_t)$, and empirical risk, $\hat{R}_T(\param_{1:T})$. 

\para{Experimental Setup} We assume that the number of items $n(\param)$ the robot attempts to grab is drawn as $\text{Pois}(C\cdot \sqrt{\param})$. The number of failed grasps conditioned on the total number of items is $d(\param)|n \sim \text{Bin}(n, C' \cdot \param)$. 
Importantly, the distributions of $n, d$, and the parameters $C, C'$ are all \emph{unknown to the agent}. 
Our conformal controller method (\textbf{CC}) adjusts the speed $\param_t$ based on the update rule from \autoref{eq:CC}.
In addition to the risk function, we also track a utility function which is the number of successful grasps $V(\param) := n(\param) - d(\param)$.

We compare our method with two baselines: A bandit algorithm running the upper confidence bound algorithm (\textbf{UCB})
\cite{lattimore_szepesvari_2020} to maximize the utility $V$ and another algorithm running the lower confidence bound algorithm (\textbf{LCB}) to minimize the loss $\loss$. 
We also add two methods with oracle access to the otherwise unknown parameters: \textbf{Oracle-Value} selects the best speed to maximize grasp success $\param^{*}_V := \arg\max_\param \E[V(\param)]$ and \textbf{Oracle-Loss} selects the best speed $\param^*_\mathcal{L}$ such that $\E[\loss(\param^*_\mathcal{L})] := \ep$. The values selected for the parameters are in \autoref{fig:factory}. 
We run all methods for a horizon $T=2000$, set $C=10$, $C'=0.2$, and the target risk is $\thresh = 0.05$ (i.e., $\leq 5\%$ failed grasp). 

\para{Results} We run the simulation $N=1000$ times, and calculate the average empirical risk and the average number of successful grasps. In \autoref{fig:factory}, we find that our method performs as well as the \textbf{Oracle-Loss}, ensuring that the empirical risk of grasps never exceeds $\thresh=0.05$, while still ensuring high throughput of successfully grasped items. \textbf{UCB} and \textbf{LCB} both violate the empirical risk threshold: UCB incurs this risk but achieves a higher number of successful grasps, while LCB is slow to learn its target, resulting in a higher risk over the time horizon.

%% file: exp_trading.tex
\subsection{Stock Trading Agent}

\begin{figure}[t!]
    \centering
    \includegraphics[width=\columnwidth,valign=t]{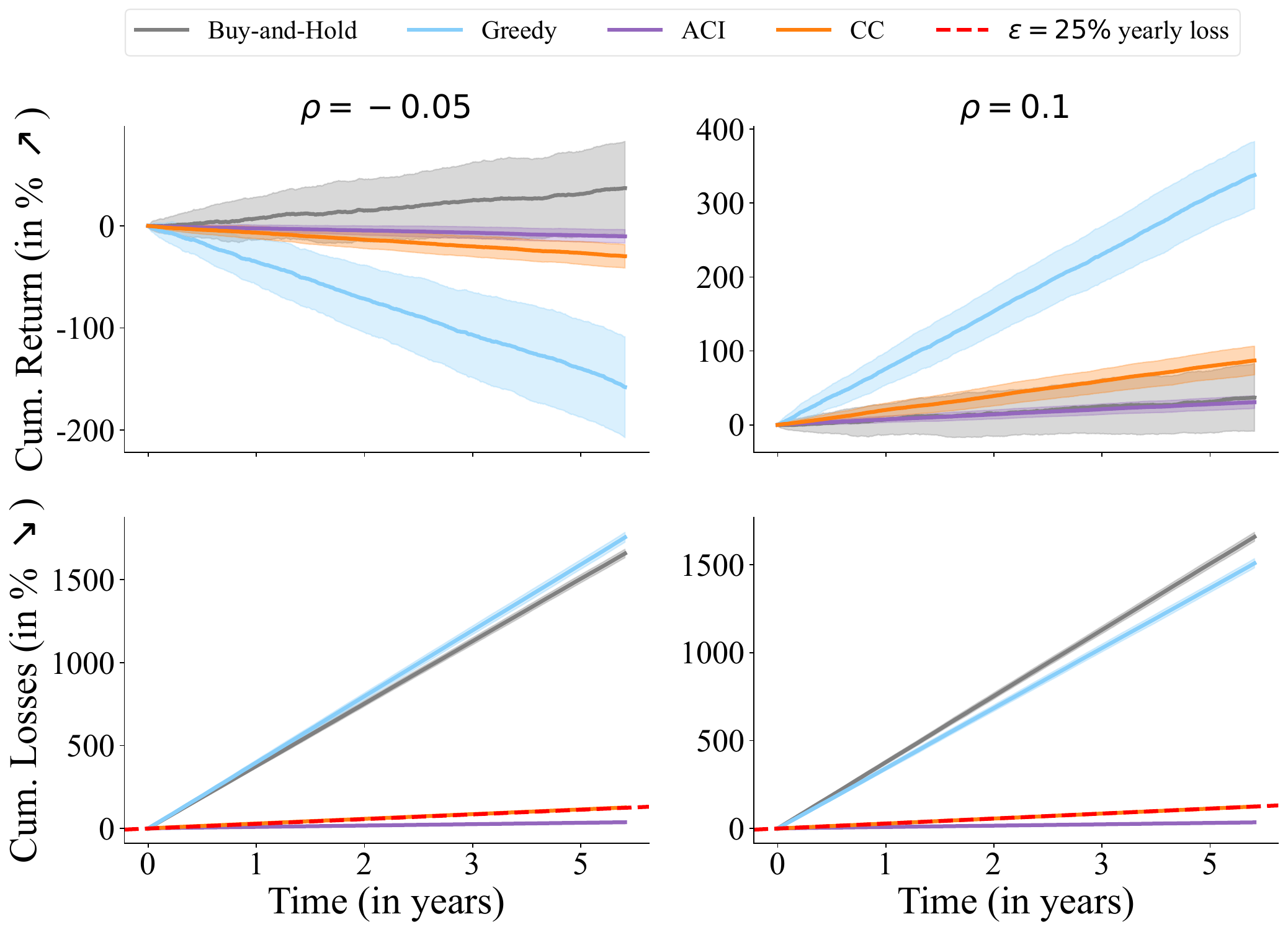}
    \caption{\textbf{Stock Trading: Quantitative Results.} All results over 5 year period. The yearly loss threshold $\thresh=25\%$. (left) Despite a poor prediction model of return (negative correlation), the CC achieves bounded loss at the user's threshold (bottom, dashed red line overlaps with orange CC line) but is not the best at keeping the return the highest. (right) With a strong prediction model on the return (positive correlation), the CC is able to achieve high yearly returns (second only to Greedy) while simultaneously respecting the loss threshold (which the Greedy violates).}
    \label{fig:trading}
    \vspace{-1em}
\end{figure}

We consider an automated trading agent that trades a stock at high frequency. We model the agent as able to either buy or short-sell the stock, with no trading cost. When buying the stock at time $t$, the agent receives return $r_t$. When short-selling the stock, the agent receives a return $-r_t$. The agent must calibrate its trading decisions so the annualized loss is at or beneath the investor's loss threshold of $\thresh \%$. 

\para{Decision Function \& Parameterization} 
At every timestep, $t$, the agent has access to the past history of returns and its own actions. The agent can use it to construct a confidence set $\hat C_\param$ where $\param$ is the conformal control variable. Given a predicted set, the agent can decide to either buy if the entire set is above zero, short-sell if the entire set is below zero, and not do anything if zero is in the set:
\begin{equation}
    \dparam_t := 
    \begin{cases}
        1 & {\rm if } \min(\hat C_\param) > 0 \\
        -1 & {\rm if } \max(\hat C_\param) < 0 \\
        0 & {\rm o.w.}
    \end{cases}
\end{equation}

\para{Risk Function}  The agent's action is $\action\in\{-1, 0, 1\}$ which incurs a loss $\loss(\action, r) := -\action \cdot r \cdot 1\{\action \cdot r <0\}$, i.e., the agent suffers a loss equal to the amount of money lost by that decision. We clip the loss to make it bounded. 

\para{Experimental Setup} We simulate stock returns using a geometric Brownian motion. We assume that we observe returns every hour, so we have $n=252(\mathrm{days}) \times 7 (\text{hours per day})$ steps per year:
\begin{equation}
    r_t := \mu \Delta + \sigma \sqrt{\Delta} Z_t \quad \text{where}\quad \Delta=1/n.
\end{equation}
We assume that at time $t-1$, the agent has access to a prediction $\hat r_t$ and we assume that the correlation $\text{corr}(r_t, \hat r_t) := \rho$. The higher $\rho$ the better the predicted returns $\hat r_t$. The predicted interval is
\[
    \hat C_\param(\hat r_t) := [\hat r_t - \sigma \sqrt{\Delta} z_{\param/2}, \hat r_t + \sigma \sqrt{\Delta} z_{1-\param/2}],
\]
where $z_\param$ is the quantile of level $\param$ of the normal distribution. 

\para{Metrics} In addition to the loss, we also measure return $V(\action,r) := \action \cdot r$ when the agent's action is $\action$.

\para{Results} We run $N=100$ simulations over five years. We set $\mu=0.08, \sigma=0.2$,
which are approximately the historical values for the S\&P 500. 
We compare our \textbf{CC} method with: the \textbf{Buy-and-Hold} strategy that simply buys the stock at each timestep, the \textbf{Greedy} strategy that buys the stock whenever the prediction is above zero and short-sells it when the prediction is below zero (equivalent to $D(\param=1)$), and \textbf{ACI} that adjusts $\param$ online using the ACI algorithm.
We set the target coverage for ACI at $90\%$ and our annualized loss threshold to be less than $\thresh = 25\%$ (the threshold per time-step is therefore $\thresh/n$). For the prediction of returns, we simulate another geometric Brownian motion,
\begin{equation}
    \hat r_t := \mu \Delta + \sigma \sqrt{\Delta} W_t \quad \text{where}\quad \text{where corr}(W_t, Z_t) = \rho.
\end{equation}
The results for the different methods are in Figure~\ref{fig:trading}. 
We plot the cumulative return and cumulative loss for all methods and for two models: $\rho=0.1$ (good model) and $\rho=-0.05$ (bad model). 
In both cases, our \textbf{CC} quickly adapts the parameter to stay below the loss threshold, while having good returns when the predictive model is good ($\rho=0.1$). 
The \textbf{Greedy} approach has more extreme returns (negative when the model is bad, positive when the model is good) with a high level of loss. 
\textbf{ACI} is highly conservative, resulting in smaller loss, significantly below the threshold. 
By being so conservative, the algorithm limits its potential gain when the predictive model is actually good. \textbf{Buy-and-hold} also has high cumulative loss as it moves with the stock, and has a more consistent return, as it is independent of the model quality.

%% file: discussion.tex
\section{Discussion \& Conclusion}
\label{sec:discussion}
In this paper, we introduce \textit{Conformal Decision Theory}, a theoretical and algorithmic framework for producing safe  decisions despite being based on imperfect machine-learning predictions. 
We have described our method in both the online adversarial setting, and also the batch exchangeable setting.
The main difference between the two is that the online algorithms we present are computationally trivial, while the batch setting can require evaluating a large amount of \emph{counterfactual} decisions (indexed by different choices of $\lambda$) on every calibration point.
Though this can be done with binary search, it still presents operational challenges.
One path for future work may be to test the method in settings where simulators or data sets can support this form of offline policy evaluation.
Another may be to develop formally valid approximations of the batch technique which preserve risk control while being more practical.
Furthermore, extensions of the batch technique to non-exchangeable settings are readily available, e.g., by use of the techniques in~\cite{farinhas2023non}, and could be evaluated.

Finally, despite $\lambda$ being 1-dimensional, our procedure can index an arbitrary set of decisions. 
Consider a set of decisions $\mathcal{D}$, a utility predictor $\hat{u}(d ; x)$ where $d \in \mathcal{D}$, and a loss predictor $\hat{\mathcal{L}}(d ; x)$, we can maximize utility subject to the constraint that our predicted loss is controlled: 
\begin{align*}
    D_t = \arg\max_{d \in \mathcal{D}} \quad &  \hat{u}(d ; x_t) \\
    \textrm{s.t.} \quad & \hat{\mathcal{L}}(d ; x_t) \leq \lambda_t.
\end{align*}
This will work as long as we revert to a safe decision if $\lambda_t \leq \lambda^{\rm safe}$; where the sequence $\lambda_t$ is defined in \autoref{eq:CC}.
However, no guarantees on utility are provided.
This topic would be a great avenue for future work, bringing conformal prediction closer to the classical statistical decision theory of Lehmann~\cite{lehmann2011some}, von Neuman and Morgenstern~\cite{von2007theory}, and others.
